\documentclass{arXiv}

\newcommand{\N}{\mathbb{N}}                                     
\newcommand{\R}{\mathbb{R}}                                     
\newcommand{\innerprod}[2]{\left\langle #1,\, #2 \right\rangle} 
\providecommand{\norm}[1]{\left\lVert #1 \right\rVert}          
\newcommand{\ts}{\hspace*{0.1em}}                               

\definecolor{boxback}{gray}{0.95}

\newcommand{\T}{\mathbf{T}}
\newcommand{\Id}{\mathrm{Id}}
\newcommand{\PSI}{\mathbf{\Psi}}
\newcommand{\XI}{\mathbf{\Xi}}
\newcommand{\mat}[3]{#1 \left| \begin{array}{@{}l@{}} \scriptstyle #3 \\[0cm] \scriptstyle #2 \end{array} \right.}

\makeatletter
\newcommand{\algrule}{\par\vskip.25\baselineskip\hrule height 0.4pt\par\vskip.75\baselineskip}
\makeatother
\algrenewcommand\algorithmicrequire{\textbf{Input:}}
\algrenewcommand\algorithmicensure{\textbf{Output:}}

\begin{document}

\title{Tensor-based algorithms for image \newline classification}

\abstract{The interest in machine learning with tensor networks has been growing rapidly in recent years. We show that tensor-based methods developed for learning the governing equations of dynamical systems from data can, in the same way, be used for supervised learning problems and propose two novel approaches for image classification. One is a kernel-based reformulation of the previously introduced MANDy (multidimensional approximation of nonlinear dynamics), the other an alternating ridge regression in the tensor-train format. We apply both methods to the MNIST and fashion MNIST data set and show that the approaches are competitive with state-of-the-art neural network-based classifiers.}



\doi{}

\author{Stefan Klus}{Department of\\Mathematics and\\Computer Science,\\Freie Universität Berlin,\\Berlin 14195, Germany}
\author{Patrick Gel\ss}{Department of\\Mathematics and\\Computer Science,\\Freie Universität Berlin,\\Berlin 14195, Germany}
\email{stefan.klus@fu-berlin.de}

\maketitle

\section{Introduction}

Tensor-based methods have become a powerful tool for scientific computing over the last years. In addition to many application areas such as quantum mechanics and computational dynamics where low-rank tensor approximations have been successfully applied, using tensor networks for supervised learning has gained a lot of attention recently. In particular the canonical format and the tensor-train format have been considered for quantum machine learning\footnote{There are different research directions in the field of quantum machine learning, we here understand it as using quantum computing capabilities for machine learning problems.} problems, see, e.g., \cite{BEYLKIN2009, NOVIKOV2015, COHEN2016}. A tensor-based algorithm for image classification using sweeping techniques inspired by the \emph{density matrix renormalization group} (DMRG)~\cite{WHITE1992} was proposed in \cite{STOUDENMIRE2016a, STOUDENMIRE2016b} and further discussed in \cite{STOUDENMIRE2018, HUGGINS2019}. Interestingly, also researchers at Google are currently developing a tensor-based machine learning framework called \emph{TensorNetwork}\footnote{\url{http://github.com/google/TensorNetwork}} \cite{ROBERTS2019, EFTHYMIOU2019}. The goal is to expedite the adoption of such methods by the machine learning community.

Our goal is to show that recently developed methods for recovering the governing equations of dynamical systems can be generalized in such a way that they can also be used for supervised learning tasks, e.g., classification problems. In order to learn the governing equations from simulation or measurement data, regression methods such as SINDy (\emph{sparse identification of nonlinear dynamics})~\cite{BRUNTON2016, RUDY2017} and its tensor-based reformulation MANDY (\emph{multidimensional approximation of nonlinear dynamics})~\cite{GELSS2019} can be applied. The main challenge is often to choose the right function space from which the system representation is learned. While SINDy and MANDy essentially select functions from a potentially large set of basis functions by applying regularized regression methods, other approaches allow nested functions and typically result in nonlinear optimization problems, which are then frequently solved using (stochastic) gradient descent. By constructing a basis comprising tensor products of simple functions (e.g., functions depending only on one variable), extremely high-dimensional feature spaces can be generated.

In this work, we explain how to compute the pseudoinverse required for solving the minimization problem directly in the tensor-train (TT) format, i.e., we replace the iterative approach from~\cite{STOUDENMIRE2016a, STOUDENMIRE2016b} by a direct computation of the least-squares solution and point out similarities with the aforementioned system identification methods. The reformulated algorithm can be regarded as a \emph{kernelized} variant of MANDy, where the kernel is based on tensor products. This is also related to quantum machine learning ideas: As pointed out in~\cite{SCHULD2019}, the basic idea of quantum computing is similar to kernel methods in that computations are performed implicitly in otherwise intractably large Hilbert spaces. While kernel methods were popular in the 1990s, the focus of the machine learning community has shifted to deep neural networks in recent years~\cite{SCHULD2019}. We will show that for simple image classification tasks, kernels based on tensor products are competitive with neural networks.

In addition to the kernel-based approach, we propose another DMRG-inspired method for the construction of TT decompositions of weight matrices containing the coefficients for the selected basis functions. Instead of computing pseudoinverses, a core-wise \emph{ridge regression} \cite{SHAWE2004} is applied to solve the minimization problem. While the approach introduced in~\cite{STOUDENMIRE2016a, STOUDENMIRE2016b} only involves tensor contractions corresponding to single images of the training data set, we use TT representations of \emph{transformed data tensors}, see~\cite{GELSS2019, NUESKE2019}, in order to include the entire training data set at once for constructing low-dimensional systems of linear equations. Combining an efficient computational scheme for the corresponding subproblems and \emph{truncated singular value decompositions}~\cite{HANSEN1987}, we call the resulting algorithm \emph{alternating ridge regression} (ARR) and discuss connections to MANDy and other regularized regression techniques.

While we describe the classification problems using the example of the iconic MNIST data set~\cite{LECUN1998} and the fashion MNIST data set~\cite{XIAO2017}, the derived algorithms can be easily applied to other classification problems. There is a plethora of kernel and deep learning methods for image classification, a list of the most successful methods for the MNIST and fashion MNIST data sets including nearest-neighbor heuristics, support vector machines, and convolutional neural networks can be found on the respective website.\!\footnote{\url{http://yann.lecun.com/exdb/mnist/}}\textsuperscript{,}\footnote{\url{http://github.com/zalandoresearch/fashion-mnist}} We will not review these methods in detail, but instead focus on relationships with data-driven methods for analyzing dynamical systems. The main contributions of this paper are:

\begin{itemize}[leftmargin=*]
  \item \emph{Extension of MANDy}: We show that the efficacy of the pseudoinverse computation in the tensor-train format can be improved by eliminating the need to left- and right-orthonormalize the tensor. While this is a straightforward modification of the original algorithm, it enables us to consider large data sets. The resulting method is closely related to kernel ridge regression. 
  \item \emph{Alternating ridge regression}: We introduce a modified TT representation of transformed data tensors for the development of a tensor-based regression technique which computes low-rank representations of coefficient tensors. We show that it is possible to obtain results which are competitive with those computed by MANDy, and, at the same time, reduce the computational costs and the memory consumption significantly.
  \item \emph{Classification of image data}: While originally designed for system identification, we apply these methods to classification problems and visualize the learned classifier, which allows us to interpret features detected in the images.
\end{itemize}

The remainder is structured as follows: In Section~\ref{sec:Prerequisites}, we will describe methods to learn governing equations of dynamical systems from data as well as a tensor-based iterative scheme for image classification and highlight their relationships. In Section~\ref{sec: tensor-based classification algorithms}, we describe how to apply MANDy to classification problems and introduce the ARR approach based on the alternating optimization of TT cores. Numerical results are presented in Section~\ref{sec:Numerical results}, followed by a brief summary and conclusion in Section~\ref{sec:Conclusion}.

\section{Prerequisites}
\label{sec:Prerequisites}

We will introduce the original MNIST and the fashion MNIST data set, which will serve as guiding examples. Afterwards SINDy and MANDy as well as tensor-based methods for image classification problems will be briefly discussed. In what follows, we will use the notation summarized in Table~\ref{tab: notation}.

\begin{table}[h]
\renewcommand{\arraystretch}{1.3}
  \caption{Notation used in this work.}
  \centering
  \begin{tabular}{ll}
    \hline
    \textbf{Symbol} & \textbf{Description} \\
    \hline
    $X = \big[ x^{(1)}, \dots, x^{(m)}\big]$ & data matrix in $\R^{d \times m}$\\
    $Y = \big[ y^{(1)}, \dots, y^{(m)}\big]$ & label matrix in $\R^{d^\prime \times m}$\\
    $n_1, \dots, n_p$  & mode dimensions of tensors \\
    $r_0, \dots, r_p$  & ranks of tensor trains \\
    $\psi_1, \dots, \psi_p$ & basis functions $\psi_\mu \colon \R^d \to \R^{n_\mu}$\\
    $\Psi_X$ / $\PSI_X$ & transformed data matrices/tensors\\
    $\Xi$ / $\XI$ & coefficient matrices/tensors\\
    \hline
  \end{tabular}
  \label{tab: notation}
\end{table}

\subsection{MNIST and fashion MNIST}

The MNIST data set \cite{LECUN1998}, see Figure~\ref{fig:MNIST_samples}\,(a), contains grayscale\footnote{The methods described below can be easily extended to color images by defining basis functions for each primary color.} images of hand-written digits and the associated labels. The data set is split into $ 60\ts000$ images for training and $ 10\ts000 $ images for testing. Each image is of size $ 28 \times 28 $. Let $ d = 784 $ be the number of pixels of one image and let the images, reshaped as vectors, be denoted by $ x^{(j)} \in \R^d $ and the corresponding labels by $ y^{(j)} \in \R^{d'} $, where $ d' = 10 $ is the number of different classes. Each label encodes a number in $ \{ 0, \dots, 9 \} $ and the entries $ y_i^{(j)} $ of the vector $ y^{(j)} $ are given by
\begin{equation}
    y_i^{(j)} =
    \begin{cases}
        1, & \text{if } x^{(j)} \text{ contains the number } i-1, \\
        0, & \text{otherwise},
    \end{cases}
\end{equation}
i.e., $ y^{(j)} = [1, 0, 0, \dots, 0]^\top $ represents $ 0 $, $ y^{(j)} = [0, 1, 0, \dots, 0]^\top $ represents $ 1 $, etc. This is also called \emph{one-hot encoding} in machine learning.

\begin{figure*}
    \centering
    \begin{subfigure}{0.3\textwidth}
        \centering
        \caption{}
        \includegraphics[height=4.5cm]{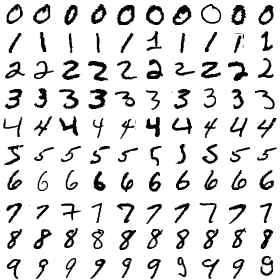}
    \end{subfigure}%
    \hspace*{7ex}
    \begin{subfigure}{0.3\textwidth}
        \centering
        \caption{}
        \includegraphics[height=4.5cm]{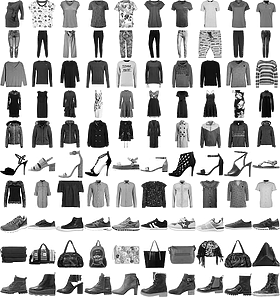}
    \end{subfigure}%
    \hspace*{7ex}
    \begin{subfigure}{0.2\textwidth}
        \centering
        \caption{}
        \renewcommand\arraystretch{1.06}
        \begin{tabular}{l@{~~}l}
            0 & T-shirt/top \\
            1 & Trousers \\
            2 & Pullover \\
            3 & Dress \\
            4 & Coat \\
            5 & Sandal \\
            6 & Shirt \\
            7 & Sneaker \\
            8 & Bag \\
            9 & Ankle boot
        \end{tabular}
    \end{subfigure}
    \caption{(a)~Samples of the MNIST data set. (b)~Samples of the fashion MNIST data set. Each row represents a different item type. (c)~Corresponding labels for the fashion MNIST data set.}
    \label{fig:MNIST_samples}
\end{figure*}

The fashion MNIST data set \cite{XIAO2017} can be regarded as a shoo-in replacement for the original data set. There are again $60\ts000$ training and $10\ts000$ test images of size $ 28 \times 28 $. Some samples are shown in Figure~\ref{fig:MNIST_samples}\,(b) and the corresponding labels in Figure~\ref{fig:MNIST_samples}\,(c). Given a picture of a clothing item, the goal now is to identify the correct category, which is encoded as described above.

\subsection{SINDy}
\label{sec: SINDy}

SINDy \cite{BRUNTON2016} was originally developed to learn the governing equations of dynamical systems from data. We will show how it can, in the same way, be used for classification problems. Consider an autonomous ordinary differential equation of the form $ \dot{x} = f(x) $, with $ f \colon \R^d \to \R^d $. Given $ m $ measurements of the state of the system, denoted by $ x^{(j)} $, $ j = 1,\dots, m $, and the corresponding time derivatives $ y^{(j)} := \dot{x}^{(j)} $, the goal is to reconstruct the function $ f $ from the measurement data. Let $X = [x^{(1)}, \dots, x^{(m)}] \in \R^{d \times m}$ and $Y = [y^{(1)}, \dots, y^{(m)}] \in \R^{d \times m}$. That is, $ d^\prime = d $ in this case. In order to represent $ f $, we select a vector-valued basis function $ \Psi \colon \R^d \to \R^n $ and define the transformed data matrix $ \Psi_X \in \R^{n \times m} $ by
\begin{equation} \label{eq:basis matrix}
    \Psi_X =
    \begin{bmatrix}
        \Psi(x^{(1)}) & \dots & \Psi(x^{(m)})
    \end{bmatrix}.
\end{equation}
Omitting sparsity constraints, SINDy then boils down to solving
\begin{equation} \label{eq:SINDy}
    \min_{\Xi} \norm{ Y - \Xi^\top \Psi_X }_F,
\end{equation}
where
\begin{equation}
    \Xi = \begin{bmatrix} \xi_1 & \dots & \xi_d \end{bmatrix} \in \R^{n \times d}
\end{equation}
is the coefficient matrix. Each column vector $ \xi_i $ then represents a function $ f_i $, i.e.,
\begin{equation}
    y_i^{(j)} \approx f_i(x^{(j)}) = \xi_i^\top \Psi(x^{(j)}).
\end{equation}
We thus obtain a model of the form $ \dot{x} = \Xi^\top \Psi(x) $, which approximates the possibly unknown dynamics. The solution of the minimization problem \eqref{eq:SINDy} with minimal Frobenius norm is given by
\begin{equation} \label{eq:pinv}
    \Xi^\top = Y \ts \Psi_X^+,
\end{equation}
where $ ^+ $ denotes the pseudoinverse, see \cite{GOLUB2013}.

\subsection{Tensor-based learning}

We will now briefly introduce the basic concepts of tensor decompositions and tensor formats as well as the tensor-based reformulation of SINDy, called MANDy, proposed in \cite{GELSS2019}. Additionally, recently introduced methods for supervised learning with tensor networks will be discussed.

\subsubsection{Tensor decompositions}

In order to mitigate the curse of dimensionality when working with tensors $\mathbf{T} \in \mathbb{R}^{n_1 \times \dots \times n_p}$, where $n_\mu \in \N$, we will exploit low-rank tensor approximations. The simplest approximation of a tensor of order $p$ is a \emph{rank-one tensor}, i.e., a tensor product of $p$ vectors given by

\begin{equation}
  \mathbf{T} = T^{(1)} \otimes T^{(2)} \otimes \dots \otimes T^{(p)},
\end{equation}
where $T^{(\mu)}$, $\mu = 1, \dots, p$, are vectors in $\R^{n_\mu}$. If a tensor is written as the sum of $r$ rank-one tensors, i.e.,
\begin{equation}
  \mathbf{T} = \sum_{k=1}^r T^{(1)}_{:, k} \otimes T^{(2)}_{:, k} \otimes \dots \otimes T^{(p)}_{:, k},
\end{equation}
with $T^{(\mu)} \in \R^{n_\mu \times r}$, this results in the so-called \emph{canonical format}. In fact, any tensor can be expressed in this format, but we are particularly interested in low-rank representations of tensors in order to reduce the storage consumption as well as the computational costs. The same requirement applies to tensors expressed in the \emph{tensor-train format} (\emph{TT format}), where a high-dimensional tensor is represented by a network of multiple low-dimensional tensors \cite{OSELEDETS2009a,OSELEDETS2011}. A tensor $\mathbf{T} \in \R^{n_1 \times \dots \times n_p}$ is said to be in the TT format if
\begin{equation}
  \mathbf{T} = \sum_{k_0=1}^{r_0} \cdots  \sum_{k_p=1}^{r_p}  \mathbf{T}^{(1)}_{k_0,:,k_1} \otimes \dots \otimes  \mathbf{T}^{(p)}_{k_{p-1},:,k_p}.
\end{equation}
The tensors $\mathbf{T}^{(\mu)} \in \R^{r_{\mu-1} \times n_\mu \times r_\mu}$ of order 3 are called \emph{TT cores}. The numbers $r_\mu$ are called \emph{TT ranks} and have a strong influence on the expressivity of a tensor train. It holds that $r_0 = r_p =1$ and $r_\mu \geq 1$ for $\mu=1, \dots, p-1$. Figure~\ref{fig: tensor train}\,(a) shows the graphical representation of a tensor train, which is also called Penrose notation, see~\cite{PENROSE1971}. 

The left- and right-unfoldings of a TT core $\T^{(\mu)}$ are given by the matrices
\begin{equation}
  \mathcal{L}_\mu = \mat{\T^{(\mu)}}{r_{\mu-1}, n_\mu}{r_\mu} \in \R^{(r_{\mu-1} \cdot n_\mu) \times r_\mu} \quad \text{and} \quad \mathcal{R}_\mu = \mat{\T^{(\mu)}}{r_{\mu-1}}{n_\mu, r_\mu} \in \R^{r_{\mu-1} \times ( n_\mu \cdot r_\mu)},
\end{equation}
respectively. Here, the indices of two modes of $\T^{(\mu)}$ are lumped into a single row or column index while the remaining mode forms the other dimension of the unfolding matrix. We call the TT core $\T^{(\mu)}$ left-orthonormal if its left-unfolding is orthonormal with respect to the rows, i.e., $\mathcal{L}_\mu^\top \cdot \mathcal{L}_\mu = \Id \in \R^{r_\mu \times r_\mu}$. Correspondingly, a core is called right-orthonormal if its right-unfolding is orthonormal with respect to the columns, i.e., \mbox{$\mathcal{R}_\mu \cdot \mathcal{R}_\mu^\top = \Id \in \R^{r_{\mu-1} \times r_{\mu-1}}$}. In Penrose notation, orthonormal components are depicted by half-filled circles, cf.~Figure~\ref{fig: tensor train}\,(b), where a tensor train with left-orthonormal cores is shown.

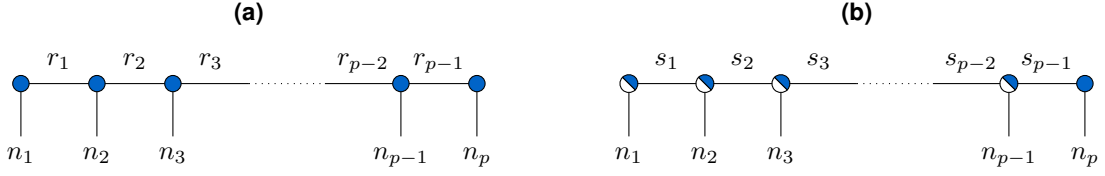
\begin{figure}
  \centering
  \begin{subfigure}{0.5\textwidth}
   \centering
   \caption{}
    \begin{tikzpicture}
        \draw[black] (0,0) -- node [label={[text depth=0]$r_1$}] {} ++ (1,0) ;
        \draw[black] (1,0) -- node [label={[text depth=0]$r_2$}] {} ++ (1,0) ;
        \draw[black] (2,0) -- node [label={[text depth=0]$r_3$}] {} ++ (1,0) ;
        \draw[black, dotted] (3,0) -- ++ (1,0) ;
        \draw[black] (4,0) -- node [label={[text depth=0]$r_{p-2}$}] {} ++ (1,0) ;
        \draw[black] (5,0) -- node [label={[text depth=0]$r_{p-1}$}] {} ++ (1,0) ;
        \draw[black] (0,0) -- node [label={[shift={(0,-0.9)}, text depth=0]$n_1$}] {} ++ (0,-0.7) ;
        \draw[black] (1,0) -- node [label={[shift={(0,-0.9)}, text depth=0]$n_2$}] {} ++ (0,-0.7) ;
        \draw[black] (2,0) -- node [label={[shift={(0,-0.9)}, text depth=0]$n_3$}] {} ++ (0,-0.7) ;
        \draw[black] (5,0) -- node [label={[shift={(0,-0.9)}, text depth=0]$n_{p-1}$}] {} ++ (0,-0.7) ;
        \draw[black] (6,0) -- node [label={[shift={(0,-0.9)}, text depth=0]$n_p$}] {} ++ (0,-0.7) ;
        \node[draw,shape=circle,fill=Blue, scale=0.65] at (0,0){};
        \node[draw,shape=circle,fill=Blue, scale=0.65] at (1,0){};
        \node[draw,shape=circle,fill=Blue, scale=0.65] at (2,0){};
        \node[draw,shape=circle,fill=Blue, scale=0.65] at (5,0){};
        \node[draw,shape=circle,fill=Blue, scale=0.65] at (6,0){};
    \end{tikzpicture}
  \end{subfigure}%
  \begin{subfigure}{0.5\textwidth}
   \centering
   \caption{}
    \begin{tikzpicture}
        \draw[black] (0,0) -- node [label={[text depth=0]$s_1$}] {} ++ (1,0) ;
        \draw[black] (1,0) -- node [label={[text depth=0]$s_2$}] {} ++ (1,0) ;
        \draw[black] (2,0) -- node [label={[text depth=0]$s_3$}] {} ++ (1,0) ;
        \draw[black, dotted] (3,0) -- ++ (1,0) ;
        \draw[black] (4,0) -- node [label={[text depth=0]$s_{p-2}$}] {} ++ (1,0) ;
        \draw[black] (5,0) -- node [label={[text depth=0]$s_{p-1}$}] {} ++ (1,0) ;
        \draw[black] (0,0) -- node [label={[shift={(0,-0.9)}, text depth=0]$n_1$}] {} ++ (0,-0.7) ;
        \draw[black] (1,0) -- node [label={[shift={(0,-0.9)}, text depth=0]$n_2$}] {} ++ (0,-0.7) ;
        \draw[black] (2,0) -- node [label={[shift={(0,-0.9)}, text depth=0]$n_3$}] {} ++ (0,-0.7) ;
        \draw[black] (5,0) -- node [label={[shift={(0,-0.9)}, text depth=0]$n_{p-1}$}] {} ++ (0,-0.7) ;
        \draw[black] (6,0) -- node [label={[shift={(0,-0.9)}, text depth=0]$n_p$}] {} ++ (0,-0.7) ;
        \node[draw,shape=semicircle,rotate=135 ,fill=white, anchor=south,inner sep=2pt, outer sep=0pt, scale=0.75] at (0,0){}; 
        \node[draw,shape=semicircle,rotate=315,fill=Blue, anchor=south,inner sep=2pt, outer sep=0pt, scale=0.75] at (0,0){};
        \node[draw,shape=semicircle,rotate=135 ,fill=white, anchor=south,inner sep=2pt, outer sep=0pt, scale=0.75] at (1,0){}; 
        \node[draw,shape=semicircle,rotate=315,fill=Blue, anchor=south,inner sep=2pt, outer sep=0pt, scale=0.75] at (1,0){};
        \node[draw,shape=semicircle,rotate=135 ,fill=white, anchor=south,inner sep=2pt, outer sep=0pt, scale=0.75] at (2,0){}; 
        \node[draw,shape=semicircle,rotate=315,fill=Blue, anchor=south,inner sep=2pt, outer sep=0pt, scale=0.75] at (2,0){};
        \node[draw,shape=semicircle,rotate=135 ,fill=white, anchor=south,inner sep=2pt, outer sep=0pt, scale=0.75] at (5,0){}; 
        \node[draw,shape=semicircle,rotate=315,fill=Blue, anchor=south,inner sep=2pt, outer sep=0pt, scale=0.75] at (5,0){};
        \node[draw,shape=circle,fill=Blue, scale=0.65] at (6,0){};
    \end{tikzpicture}
  \end{subfigure}%
  \caption{Graphical representation of tensor trains: (a) A core is depicted by a circle with different arms indicating the modes of the tensor and the rank indices. The first and the last TT core are regarded as matrices due to the fact that $r_0 = r_p = 1$. (b) Left-orthonormalized tensor train obtained by, e.g., sequential SVDs. Note that the TT ranks may change due to orthonormalization, e.g., when using (reduced/truncated) SVDs.}
  \label{fig: tensor train}
\end{figure}

A given TT core can be left- or right-orthonormalized, respectively, by computing a singular value decomposition (SVD) of its unfolding. For instance, the components of an SVD of the form $ \mathcal{L}_\mu = U \cdot \Sigma \cdot V^\top$ can be interpreted as a left-orthonormalized version of $\T^{(\mu)}$ coupled with the matrices $\Sigma$ and $V^\top$. When we talk about, e.g., left-orthonormalization of the cores of a tensor train, we mean the application of sequential SVDs from left to right (also called \emph{HOSVD}, cf.~\cite{OSELEDETS2009b}) where $U$ builds the updated core, while the non-orthonormal part $\Sigma \cdot V^\top$ is contracted with the subsequent TT core. As described in \cite{KLUS2018, GELSS2019, NUESKE2019}, left- and right-orthonormalization can be used to construct pseudoinverses of tensors. The general idea is to construct a global SVD of a given tensor train by left- and right-orthonormalizing its cores. However, in Section~\ref{sec: Kernel-based MANDy} we will exploit the structure of transformed data tensors as introduced in \cite{GELSS2019} to propose a different method for the construction of pseudoinverses which significantly reduces the computational effort.

We also represent TT cores as two-dimensional arrays containing vectors as elements. In this notation, a single core of a tensor train $\mathbf{T} \in \R^{n_1 \times \dots \times n_p}$ is written as
\begin{equation}
  \left\llbracket \mathbf{T}^{(\mu)} \right\rrbracket = 
  \left\llbracket~
  \begin{matrix}
    \mathbf{T}^{(\mu)}_{1,:,1} & \cdots & \mathbf{T}^{(\mu)}_{1,:,r_\mu} \\
    & & \\
    \vdots & \ddots & \vdots \\
    & & \\
    \mathbf{T}^{(\mu)}_{r_{\mu-1},:,1} & \cdots & \mathbf{T}^{(\mu)}_{r_{\mu-1},:,r_\mu}
  \end{matrix}~\right\rrbracket.
  \label{eq: core notation - single core}
\end{equation}
Then, the expression $\mathbf{T} = \left\llbracket \mathbf{T}^{(1)}\right\rrbracket \otimes \dots \otimes \left\llbracket \mathbf{T}^{(p)}\right\rrbracket$ is used for representing tensor trains $\mathbf{T}$, cf.\ \cite{GELSS2016, GELSS2017, GELSS2019}. 

\subsubsection{MANDy}

MANDy \cite{GELSS2019} is a tensorized version of SINDy and constructs counterparts of the transformed data matrices \eqref{eq:basis matrix} directly in the TT format. Two different types of decompositions, namely the \emph{coordinate-} and the \emph{function-major decomposition}, were introduced in \cite{GELSS2019}. In \cite{NUESKE2019}, the technique for the construction of the transformed data tensors was generalized to arbitrary lists of basis functions. This will be explained in more detail in Section~\ref{sec: basis decomposition}. Given data matrices $X, Y \in \R^{d \times m}$ and basis functions $\psi_\mu \colon \R^d \to \R^{n_\mu}$, $\mu = 1, \dots, p$, the tensor-based representation of the corresponding transformed data tensors $\PSI_X \in \R^{n_1 \times \dots \times n_p \times m}$ enables us to solve the reformulated minimization problem
\begin{equation}\label{eq: MANDy optimization problem}
  \min_{\XI} \norm{Y - \XI^\top \PSI_X}_F
\end{equation}
so that the coefficients are given in form of a tensor train $\XI \in \R^{n_1 \times \dots \times n_p \times d}$, cf.~Section~\ref{sec: SINDy}. Instead of identifying the governing equations of dynamical systems from data, see \cite{GELSS2019}, we seek to classify images using MANDy. The only difference is that $ \PSI_X $ now contains the transformed images and $ Y $ the corresponding labels. Since the matrix $Y$ may have different dimensions than $X$, i.e., $Y \in \R^{d^\prime \times m}$, the aim is to find the optimal solution of \eqref{eq: MANDy optimization problem} in form of a tensor train \mbox{$\XI \in \R^{n_1 \times \dots \times n_p \times d^\prime}$}. We will discuss the explicit representation of transformed data tensors and their pseudoinversion in Section~\ref{sec: tensor-based classification algorithms}.

\subsubsection{Supervised learning with tensor networks}
\label{sec: supervised}

It has been shown in \cite{STOUDENMIRE2016a, STOUDENMIRE2016b} that tensor-based optimization schemes can be adapted to supervised learning problems. A given input vector $ x $ is mapped into a higher-dimensional space using a feature map $ \PSI $ before being classified by a \emph{decision function} $ f\colon \R^d \to \R^{d^\prime} $ of the form
\begin{equation} \label{eq:decision function}
    f(x) = \XI^\top \ts \PSI(x),
\end{equation}
where $ \XI $ is a coefficient tensor in TT format. The $ i $th entry of the vector $ f(x) $ then represents the likelihood that the image $ x $ belongs to the class with label $ i-1 $. The transformation defined in \cite{STOUDENMIRE2016a, STOUDENMIRE2016b} reads as follows:
\begin{equation} \label{eq:Psi}
    \mathbf{\PSI}(x) = 
    \begin{bmatrix}
        \cos(\alpha \ts x_1) \\ \sin(\alpha \ts x_1)
    \end{bmatrix} \otimes 
    \begin{bmatrix}
        \cos(\alpha \ts x_2) \\ \sin(\alpha \ts x_2)
    \end{bmatrix} \otimes
    \dots \otimes
    \begin{bmatrix}
        \cos(\alpha \ts x_d) \\ \sin(\alpha \ts x_d)
    \end{bmatrix},
\end{equation}
where $ \alpha $ is a parameter. It turns out that the originally proposed choice of $ \alpha = \frac{\pi}{2} $ is often not optimal. This will be discussed in more detail below. The function $\PSI$ assigns each pixel of the image a two-dimensional vector, inspired by the spin-vectors encountered in quantum mechanics~\cite{STOUDENMIRE2016b}. It was illustrated in \cite{SCHULD2019} how such a transformation can be implemented as a quantum feature map, where the information is encoded in the amplitudes of qubits. Embedding data into quantum Hilbert spaces might be interesting in cases where the quantum device evaluates kernels faster or where kernels cannot be simulated by classical computers anymore \cite{SCHULD2019}.

Due to the tensor structure, $ \PSI(x) $ is a tensor with $ 2^d $ entries, which, for the original MNIST image size amounts to $ n \approx 10^{236} $ basis functions. In \cite{STOUDENMIRE2016a, STOUDENMIRE2016b}, the image size is first reduced to $ 14 \times 14 $ pixels by averaging groups of four pixels, which then results in ``only'' $ n \approx 10^{59} $ basis functions. Thus, storing the full coefficient matrix is clearly infeasible since $ \XI \in \R^{2 \times \dots \times 2 \times d^\prime} \cong \R^{n \times d^\prime}$. Here, $d^\prime$ appears as an additional tensor index since the decision function is computed for all $ d^\prime $ labels simultaneously.

In order to learn the tensor $ \XI $ from training data, a DMRG/ALS-related algorithm (cf.~\cite{WHITE1992, HOLTZ2012}) that sweeps back and forth along the cores and iteratively minimizes the cost function
\begin{equation} \label{eq:Tensor minimization problem}
    \min_{\XI} \sum_{j=1}^m \norm{ y^{(j)} - \XI^\top \ts \PSI(x^{(j)}) }_2^2
\end{equation}
is devised. The suggested algorithm varies two neighboring cores at the same time, which allows for adapting the tensor ranks, and computes an update using a gradient descent step. The tensor ranks are reduced by truncated SVDs to control the computational costs. The truncation of the TT ranks can also be interpreted as a form of regularization. For more details, we refer to \cite{STOUDENMIRE2016a, STOUDENMIRE2016b}. 

Different techniques to improve the original algorithm presented in \cite{STOUDENMIRE2016a} were proposed. In \cite{LIU2018}, the image data is preprocessed using a discrete cosine transformation and the ordering of the pixels is optimized in order to reduce the ranks. In \cite{EFTHYMIOU2019}, the DMRG-based sweeping method was replaced by a stochastic gradient descent approach, where the gradient is computed with the aid of automatic differentiation. Furthermore, it was shown that GPUs allow for an efficient solution of such problems.

\section{Tensor-based classification algorithms}
\label{sec: tensor-based classification algorithms}

We will now describe two different tensor-based classification approaches. First, we show how to combine MANDy with kernel-based regression techniques in order to derive an efficient method for the computation of the pseudoinverse of the transformed data tensor. Then, a classification algorithm based on the alternating optimization of the TT cores of the coefficient tensor is proposed. 

\subsection{Basis decomposition}
\label{sec: basis decomposition}

As above, let $x \in \R^d$ be a vector and $\psi_{\mu} \colon \R^d \rightarrow \R^{n_\mu}$, $\mu = 1, \dots, p$, basis functions. We consider the rank-one tensors
\begin{equation}
  \label{eq: basis decomposition}
  \PSI(x) = \psi_1(x) \otimes \dots \otimes \psi_p(x) = \begin{bmatrix} \psi_{1,1} (x) \\ \vdots \\ \psi_{1, n_1}(x) \end{bmatrix} \otimes \dots \otimes \begin{bmatrix} \psi_{p,1} (x) \\ \vdots \\ \psi_{p,n_p}(x) \end{bmatrix} \in \R^{n_1 \times n_2 \times \dots \times n_p}.
\end{equation}

For $m$ different vectors stored in a data matrix $X = [x^{(1)}, \dots, x^{(m)}] \in \R^{d \times m}$, we want to construct transformed data tensors $\PSI_X \in \R^{n_1 \times \dots \times n_p \times m}$ with $(\PSI_X)_{:, \dots , :, j} = \PSI(x^{(j)})$. In \cite{GELSS2019, NUESKE2019}, this was achieved by multiplying (with the aid of the tensor product) the rank-one decompositions given in \eqref{eq: basis decomposition} for all vectors $x^{(1)}, \dots, x^{(m)}$ by additional unit vectors and subsequently summing them up. The transformed data tensor can then be represented using the following canonical/TT decompositions:
\begin{equation}\label{eq: transformed data tensor}
  \begin{split}
    \PSI_X  &= \sum_{j=1}^m \PSI\left(x^{(j)}\right) \otimes e_j \\
    &= \sum_{j=1}^m \psi_1\left(x^{(j)}\right) \otimes \dots \otimes \psi_p \left(x^{(j)} \right) \otimes e_j \\
    &= \left \llbracket \begin{matrix}
    \psi_1 \left(x^{(1)}\right) & \cdots & \psi_1 \left(x^{(m)}\right)
    \end{matrix} \right\rrbracket \otimes
    \left\llbracket\begin{matrix}
    \psi_2 \left(x^{(1)}\right) & & 0 \\
    & \ddots & \\
    0 & & \psi_2 \left(x^{(m)}\right)
    \end{matrix} \right\rrbracket \otimes \cdots \\
    & \qquad \cdots \otimes
    \left \llbracket \begin{matrix}
    \psi_p \left(x^{(1)}\right) & & 0 \\
    & \ddots & \\
    0 & & \psi_p \left(x^{(m)}\right) 
    \end{matrix} \right\rrbracket \otimes
    \left \llbracket \begin{matrix}
    e_1 \\
    \vdots \\
    e_m
    \end{matrix} \right \rrbracket\\
    &= \left \llbracket \PSI_X^{(1)} \right \rrbracket \otimes \dots \otimes \left \llbracket \PSI_X^{(p+1)} \right \rrbracket,
  \end{split}
\end{equation}
where $e_j$, $j = 1, \dots , m$, denote the unit vectors of the standard basis in the $m$-dimensional Euclidean space. An entry of $\PSI_X$ is given by
\begin{equation}
  \begin{split}
    \left( \PSI_X \right)_{i_1, \dots, i_p, j} &= \psi_{1,i_1}\left(x^{(j)}\right) \cdot \ldots \cdot \psi_{p,i_p}\left(x^{(j)}\right),
  \end{split}
\end{equation}
for $1 \leq i_k \leq n_k$ and $1 \leq j \leq m$. Thus, the matrix-based counterpart of $\PSI_X$, see \eqref{eq:basis matrix}, would be given by the mode-$p$ unfolding
\begin{equation}
  \Psi_X = \mat{\PSI_X}{n_1, \dots, n_p}{m}.
\end{equation}
That is, the modes $n_1 , \dots, n_p$ represent row indices of the unfolding, and mode $m$ is the column index. However, for the purpose of this paper, we modify the representation of our transformed data tensors. First, realize that the last core of the TT representation in \eqref{eq: transformed data tensor} can be neglected since it is only a reshaped identity matrix. The result is then a tensor network with an ``open arm'' which can be regarded as a tensor train with an additional column mode located at the last core, see Figure~\ref{fig: TDT}\,(a). Second, this additional mode can be shifted to any TT core of the decomposition. This is shown in Figure~\ref{fig: TDT}\,(b). We will benefit from these modifications in Section~\ref{sec: ALS approach} when constructing the subproblems for the ALS-inspired approach. Consider the TT decomposition $\widehat{\PSI}_X$ given by
\begin{equation}
  \begin{split}
    \widehat{\PSI}_X
    &= \left \llbracket \PSI_X^{(1)} \right \rrbracket \otimes \dots \otimes \left \llbracket \PSI_X^{(p-1)} \right \rrbracket \otimes
    \left \llbracket \begin{matrix}
    \psi_p \left(x^{(1)}\right) \\
    \vdots \\
    \psi_p \left(x^{(m)}\right)
    \end{matrix} \right \rrbracket.
  \end{split}
\end{equation}
Note that this tensor is an element of the tensor space $\R^{n_1 \times \dots \times n_p}$, i.e., $\widehat{\PSI}_X$ has no additional column dimension and it holds that
\begin{equation}
  \mat{\widehat{\PSI}_X}{n_1, \dots, n_p}{~} = \Psi_X \cdot [1, \dots, 1]^\top.
\end{equation}

\begin{figure}
  \centering
  \begin{subfigure}{0.5\textwidth}
    \centering
    \caption{}
    \begin{tikzpicture}
        \draw[black] (0,0) -- ++ (1,0) ;
        \draw[black] (1,0) -- ++ (1,0) ;
        \draw[black] (2,0) -- ++ (0.66,0) ;
        \draw[black, dotted] (2.66,0) -- ++ (0.66,0) ;
        \draw[black] (3.33,0) -- ++ (0.66,0) ;
        \draw[black] (4,0) -- ++ (2,0) ;
        \draw[black] (0,0) -- node [label={[shift={(0,-1)}]$n_1$}] {} ++ (0,-0.7) ;
        \draw[black] (1,0) -- node [label={[shift={(0,-1)}]$n_2$}] {} ++ (0,-0.7) ;
        \draw[black] (2,0) -- node [label={[shift={(0,-1)}]$n_3$}] {} ++ (0,-0.7) ;
        \draw[black] (4,0) -- node [label={[shift={(0,-1)}]$n_{p-2}$}] {} ++ (0,-0.7) ;
        \draw[black] (5,0) -- node [label={[shift={(0,-1)}]$n_{p-1}$}] {} ++ (0,-0.7) ;
        \draw[black] (6,0) -- node [label={[shift={(0,-1)}]$n_{p}$}] {} ++ (0,-0.7) ;
        \draw[black] (6,0) -- node [label={[shift={(0,0.2)}]$m$}] {} ++ (0,+0.7) ;
        \node[draw,shape=circle,fill=Blue, scale=0.65] at (0,0){};
        \node[draw,shape=circle,fill=Blue, scale=0.65] at (1,0){};
        \node[draw,shape=circle,fill=Blue, scale=0.65] at (2,0){};
        \node[draw,shape=circle,fill=Blue, scale=0.65] at (4,0){};
        \node[draw,shape=circle,fill=Blue, scale=0.65] at (5,0){};
        \node[draw,shape=circle,fill=Green, scale=0.65] at (6,0){};
    \end{tikzpicture}
  \end{subfigure}%
  \begin{subfigure}{0.5\textwidth}
    \centering
    \caption{}
    \begin{tikzpicture}
        \draw[black] (0,0) -- ++ (0.66,0) ;
        \draw[black, dotted] (0.66,0) -- ++ (0.66,0) ;
        \draw[black] (1.33,0) -- ++ (0.66,0) ;
        \draw[black] (2,0) -- ++ (1,0) ;
        \draw[black] (3,0) -- ++ (1,0) ;
        \draw[black] (4,0) -- ++ (0.66,0) ;
        \draw[black, dotted] (4.66,0) -- ++ (0.66,0) ;
        \draw[black] (5.33,0) -- ++ (0.66,0) ;
        \draw[black] (0,0) -- node [label={[shift={(0,-1)}]$n_1$}] {} ++ (0,-0.7) ;
        \draw[black] (2,0) -- node [label={[shift={(0,-1)}]$n_{i-1}$}] {} ++ (0,-0.7) ;
        \draw[black] (3,0) -- node [label={[shift={(0,-1)}]$n_{i}$}] {} ++ (0,-0.7) ;
        \draw[black] (4,0) -- node [label={[shift={(0,-1)}]$n_{i+1}$}] {} ++ (0,-0.7) ;
        \draw[black] (6,0) -- node [label={[shift={(0,-1)}]$n_{p}$}] {} ++ (0,-0.7) ;
        \draw[black] (3,0) -- node [label={[shift={(0,0.2)}]$m$}] {} ++ (0,+0.7) ;
        \node[draw,shape=circle,fill=Blue, scale=0.65] at (0,0){};
        \node[draw,shape=circle,fill=Blue, scale=0.65] at (2,0){};
        \node[draw,shape=circle,fill=Green, scale=0.65] at (3,0){};
        \node[draw,shape=circle,fill=Blue, scale=0.65] at (4,0){};
        \node[draw,shape=circle,fill=Blue, scale=0.65] at (6,0){};
    \end{tikzpicture}
  \end{subfigure}%
  \caption{TT representation of transformed data tensors: (a) As in \cite{GELSS2019}, the first $p$ cores (blue circles) are given by \eqref{eq: transformed data tensor}. The direct contraction of the two last TT cores in \eqref{eq: transformed data tensor} can be regarded as an operator-like TT core with a row and column mode (green circle). (b) The additional column mode can be shifted to any of the $p$ TT cores.}
  \label{fig: TDT}
\end{figure}
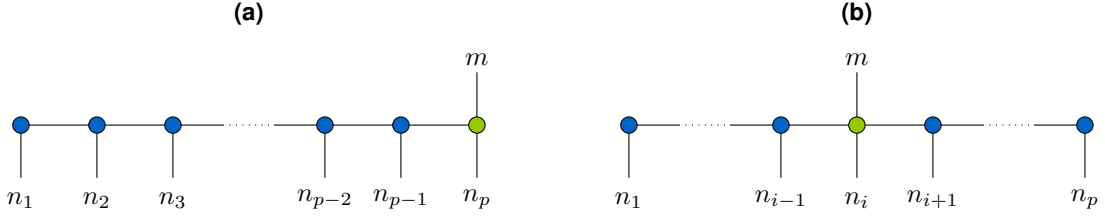

\noindent Now, we define $\widehat{\PSI}_{X,\mu} \in \R^{n_1 \times \dots \times n_p \times m}$ to be the tensor derived from $\widehat{\PSI}_X$ by replacing the $\mu$th core by
\begin{equation}\label{eq: 4D core}
  \widehat{\PSI}_{X,\mu}^{(\mu)}  = \left \llbracket%
  \begin{matrix}
    \begin{bmatrix}
      | & 0 & \cdots & 0 & \, \\
      \raisebox{0.25em}{$\psi_{\mu} (x_1)$} & \vdots &  & \vdots & \, \\
      | & 0 & \cdots & 0 & \, 
    \end{bmatrix}
    & & 0 \\
    & \ddots & \\
    0 & & 
    \begin{bmatrix}
       \, & 0 & \cdots & 0 & |\\
       \, & \vdots &  & \vdots & \raisebox{0.25em}{$\psi_{\mu} (x_m)$}\\
       \, & 0 & \cdots & 0 & |
    \end{bmatrix}
  \end{matrix}%
  \right\rrbracket \in \R^{m \times n_\mu \times m \times m},
\end{equation}
where the outer modes correspond to the rank dimensions while the inner modes represent the dimensions of the matrices. Analogously, for the first and the last core of $\widehat{\PSI}_{X,\mu}$ the non-diagonal core structure has to be used. The $4$-dimensional TT core \eqref{eq: 4D core} naturally represents a component of a TT operator. In what follows, we will not need to store the whole TT core given in \eqref{eq: 4D core}. Otherwise this would mean that we have to save $m^3 \cdot n$ scalar entries (not using a sparse format). However, from a theoretical point of view, $\PSI_{X}$ in Figure~\ref{fig: TDT}~(a) and$\widehat{\PSI}_{X,\mu}$ in Figure~\ref{fig: TDT}~(b) represent the same tensor in $\R^{n_1 \times \dots \times n_p \times m}$, see Appendix~\ref{app: representation of transformed data tensors}.

\subsection{Kernel-based MANDy}
\label{sec: Kernel-based MANDy}

Given a training set $X \in \R^{d \times m}$, the corresponding label matrix $Y \in \R^{d^\prime \times m}$, and a set of basis functions $\psi_{\mu} \colon \R^d \rightarrow \R^{n_\mu}$, $\mu = 1, \dots, p$, we exploit the canonical representation of $\PSI_X$ given in \eqref{eq: transformed data tensor} for kernel-based MANDy. The aim is to solve the optimization problem \eqref{eq: MANDy optimization problem}, i.e., we try to find a coefficient tensor $\XI \in \R^{n_1 \times \dots \times n_p \times d^\prime}$ such that $\XI^\top \PSI_X$ is as close as possible to the corresponding label matrix $Y \in \R^{d^\prime \times m}$. The solution of \eqref{eq: MANDy optimization problem} with minimal Frobenius norm is given by $\XI^\top = Y \PSI_X^+$, cf.~\eqref{eq:pinv}. Note that, compared to standard SINDy/MANDy, the matrix $Y$ here does not necessarily have the same dimensions as $X$. Due to potentially large ranks of the transformed data tensor $\PSI_X$, the direct computation of the pseudoinverse using left- and right-orthonormalization as proposed in \cite{GELSS2019} would be computationally expensive. However, using the identity $\PSI_X^+ = (\PSI_X^\top \PSI_X)^+ \PSI_X^\top$, we can rewrite the coefficient tensor as
\begin{equation}
  \XI^\top = Y \left(\PSI_X^\top \PSI_X\right)^+ \PSI_X^\top.
\end{equation}
The contraction of $\PSI_X^\top$ and $\PSI_{X}$ yields a Gram matrix $G \in \R^{m \times m}$ whose entries are given by the resulting kernel function $k(x, x^\prime) = \innerprod{\PSI(x)}{\PSI(x^\prime)}$, i.e.,
\begin{equation}
  G_{i,j} = k\left(x^{(i)}, x^{(j)}\right) = \innerprod{\PSI\left (x^{(i)} \right)}{\PSI\left(x^{(j)} \right)}.
\end{equation}
Note that due to the tensor structure of $ \PSI_X$, we obtain
\begin{equation}
  k\left(x^{(i)}, x^{(j)}\right)= \prod_{\mu=1}^p \innerprod{\psi_\mu\left(x^{(i)} \right)}{\psi_\mu\left(x^{(j)} \right)},
\end{equation}
i.e., a product of $p$ local kernels.
\begin{remark}
  For the basis functions defined in~\eqref{eq:Psi}, this can be simplified to
\begin{equation}
    k(x, x^\prime) = \prod_{i=1}^d \cos\left(\alpha (x_i-x^\prime_i)\right),
\end{equation}
which is a product of cosine kernels, cf.~\cite{STOUDENMIRE2016a}.
\end{remark}
The product structure of the kernel allows us to compute the Gram matrix $G$ as a Hadamard product (denoted by $\odot$) of $p$ matrices, that is,
\begin{equation}\label{eq: MANDy - Hadamard 1}
  G = \Theta_1 \odot \Theta_2 \odot \dots \odot \Theta_p,
\end{equation}
where $\Theta_\mu \in \R^{m \times m}$ is given by
\begin{equation}\label{eq: MANDy - Hadamard 2}
  \Theta_\mu = \left[\psi_\mu \left(x^{(1)} \right) , \dots, \psi_\mu \left(x^{(m)} \right) \right]^\top \cdot \left[\psi_\mu \left(x^{(1)}\right), \dots , \psi_\mu \left(x^{(m)}\right) \right].
\end{equation}
We now define $ Z := Y \ts G^+ \in \R^{d^\prime \times m} $, which can be obtained by solving the system $ Z \ts G = Y $ (in the least-squares sense if $ G $ is singular). The decision function $ f $, cf.~\eqref{eq:decision function}, is then given by
\begin{equation} \label{eq:kernel classification}
    f(x) = \underbrace{Z \PSI_X^\top}_{=: \XI^\top} \PSI(x)
         = Z
         \begin{bmatrix}
             k(x_1, x) \\
             \vdots \\
             k(x_m, x)
         \end{bmatrix}
\end{equation}
and again only requires kernel evaluations. As above, we can use a sequence of Hadamard products to compute $\PSI_X^\top \PSI(x)$. The classification problem can thus be solved as summarized in Algorithm~\ref{alg:kernel classification}.

\begin{mdframed}[backgroundcolor=boxback,hidealllines=true]
\vspace*{-2ex}
\begin{algorithm}[H]
  \caption{Kernel-based MANDy for classification.}
  \label{alg:kernel classification}
  \begin{spacing}{1.2}
  \begin{algorithmic}[1]
    \Require Training set $X$ and label matrix $Y$, test set $\tilde{X}$, basis functions.
    \Ensure Label matrix $\tilde{Y}$.
    \algrule
    \State Compute $ G $ using \eqref{eq: MANDy - Hadamard 1} and \eqref{eq: MANDy - Hadamard 2}.
    \State Solve $ Z \ts G = Y $.\label{alg:kernel classification - sle}
    \State Define the decision function $f$ using \eqref{eq:kernel classification}.
    \State Apply $f$ to every vector $\tilde{x}$ in the test set, and store the resulting vectors $\tilde{y}$ in the matrix $\tilde{Y}$.
    \State The index $i$ of the largest entry of $ \tilde{y} $ determines the detected label, i.e., set $\tilde{y} = e_i$.
  \end{algorithmic}
  \end{spacing}
\end{algorithm}
\vspace*{-2ex}%
\end{mdframed}

We could also replace the pseudoinverse $ G^+ $ by the regularized inverse $ (G + \varepsilon \ts \Id)^{-1} $, where $ \varepsilon $ is the regularization parameter, which would lead to a slightly different system of linear equations. However, for the numerical experiments in Section~\ref{sec:Numerical results}, we do not use regularization. Algorithm~\ref{alg:kernel classification} is equivalent to \emph{kernel ridge regression} (see, e.g., \cite{SHAWE2004}) with a tensor-product kernel. This is not surprising since we are solving simple least-squares problems.

\begin{remark}
Note that the kernel does not necessarily have to be based on tensor products of basis functions for this method to work, we could also simply use, e.g., a Gaussian kernel, which for the MNIST data set leads to slightly lower but similar classification rates. Tensor-based kernels, however, have an exponentially large yet explicit feature space representation and additional structure that could be exploited to speed up computations. Moreover, the kernel-based algorithm outlined above can in the same way be applied to time-series data to learn governing equations in potentially infinite-dimensional feature spaces.
\end{remark}

Compared to the method proposed in \cite{STOUDENMIRE2016a, STOUDENMIRE2016b}, the advantage of our approach, which can be regarded as a kernel-based formulation of MANDy (or SINDy), is that we can compute a closed-form solution without necessitating any iterations or sweeps. However, even though this approach for classification problems computes an optimal solution of the minimization problem \eqref{eq: MANDy optimization problem}, the runtime as well as the memory consumption of the algorithm depend crucially on the size of the training data set (and also the number of labels) and the resulting coefficient tensor $\XI$ has no guaranteed low-rank structure. We will now propose an alternating optimization method which circumvents this problem.

\subsection{Alternating ridge regression (ARR)}
\label{sec: ALS approach}

In what follows, we will use the TT representation illustrated in Figure~\ref{fig: TDT}\,(b) for the transformed data tensor $\PSI_X \in \R^{n_1 \times \dots \times n_p \times m}$. Even though we do not consider a TT operator, the proposed approach is closely related to the DMRG method~\cite{WHITE1992}, also called \emph{alternating linear scheme} (ALS)~\cite{HOLTZ2012}. As in \cite{STOUDENMIRE2016a, STOUDENMIRE2016b}, the idea here is to compute a low-rank TT approximation of the coefficient tensor $\XI$ by an alternating scheme. That is, a low-dimensional system of linear equations has to be solved for each TT core. Our approach is outlined in Algorithm~\ref{alg: ARR}.

\begin{mdframed}[backgroundcolor=boxback,hidealllines=true]
\vspace*{-2ex}
\begin{algorithm}[H]
  \caption{Alternating ridge regression (ARR) for classification.}
  \label{alg: ARR}
  \begin{spacing}{1.2}
  \begin{algorithmic}[1]
    \Require Training set $X$ and label matrix $Y$, test set $\tilde{X}$, basis functions, initial guesses.
    \Ensure Label matrix $\tilde{Y}$.
    \algrule
    \For{$i=1, \dots, d^\prime$}\hfill\textit{(parallelizable)}
    \State Define $w = Y_{i,:} = [Y^{(1)}_i , \dots , Y^{(m)}_i]$.
    \State Define initial guess $\XI_i$ and right-orthonormalize.
    \State Compute right stack $Q_p, \dots, Q_1$.\label{alg: ARR - right stack 1}
    \For{$\mu = 1, \dots, p-1$}\hfill\textit{(first half sweep)}\label{alg: ARR - start iteration}
    \State Compute $P_\mu$.
    \State Construct micro-matrix $M_\mu$ from $P_\mu, \widehat{\PSI}_{X,\mu}^{(\mu)}, Q_\mu$.
    \State Determine truncated SVD solution of $\min_v \norm{w - v M_\mu}_2$.\label{alg: ARR - subproblem 1}
    \State Apply QR decomposition to extract updated core.
    \EndFor
    \For{$\mu = p, \dots, 1$}\hfill\textit{(second half sweep)}
    \State Compute $Q_\mu$.\label{alg: ARR - right stack 2}
    \State Construct micro-matrix $M_\mu$ from $P_\mu, \widehat{\PSI}_{X,\mu}^{(\mu)}, Q_\mu$.
    \State Determine truncated SVD solution of $\min_v \norm{w - v M_\mu}_2$.\label{alg: ARR - subproblem 2}
    \If{$\mu>1$}
    \State Apply QR decomposition to extract updated core.
    \Else
    \State Set the updated core to a reshape of $v$.\label{alg: ARR - end iteration}
    \EndIf
    \EndFor
    \State Repeat lines \ref{alg: ARR - start iteration}--\ref{alg: ARR - end iteration} to increase accuracy (if needed).
    \EndFor
    \State Define $\XI$ using \eqref{eq: ARR - coefficient tensor} and set $ y = f(x) $ using \eqref{eq:decision function}.
    \State The index of the largest entry of $ y $ determines the detected label, see Algorithm~\ref{alg:kernel classification}.
  \end{algorithmic}
  \end{spacing}
\end{algorithm}
\vspace*{-2ex}%
\end{mdframed}

First, note that instead of solving the minimization problem \eqref{eq: MANDy optimization problem} we can also find separate solutions of
\begin{equation}\label{eq: ARR - minimization problem}
    \min_{\XI_i} \norm{Y_{i,:} - \XI^\top_i \PSI_X}_2
\end{equation}
for each row of $Y$. Since these systems can be solved independently, Algorithm~\ref{alg: ARR} can be easily parallelized. We then use a DMRG/ALS-inspired scheme to split the optimization problem \eqref{eq: ARR - minimization problem} into $p$ subproblems. The micro-matrix $M_\mu$ of such a subproblem can be built from three different parts, namely $\widehat{\PSI}_{X,\mu}^{(\mu)}$, $P_\mu$, and $Q_\mu$. The latter are both collected in a left and right stack in order to avoid repetitive computations. Note that $P_\mu$ is determined by contracting $P_{\mu-1}$ with the $(\mu-1)$th cores of $\XI_i$ and $\widehat{\PSI}_X$. Analogously, $Q_\mu$ is build from $Q_{\mu+1}$ and the $(\mu+1)$th cores of $\XI_i$ and $\widehat{\PSI}_X$.  During the first half sweep of Algorithm~\ref{alg: ARR}, we only have to compute the matrices $P_\mu$ since the used matrices $Q_\mu$ are not based on any updated cores. Afterwards, the matrices $Q_\mu$ are (re-)computed during the second half. See \cite{HOLTZ2012} for further details and Figure~\ref{fig: ALS} for a graphical illustration of the construction of the subproblems and the extraction of the optimized core. Note that it is not necessary to store the (sparse) core $\widehat{\PSI}_{X,\mu}^{(\mu)}$ in its full representation as a $4$-dimensional array in order to construct the matrix $M_\mu$. By using, e.g., NumPy's \texttt{einsum} the TT core can be replaced a (dense) matrix containing the corresponding function evaluations.

\begin{figure}
  \centering
  \begin{subfigure}{0.6\textwidth}
    \centering
    \caption{}
     \begin{tikzpicture}
        \draw[white] (3,1.65) -- ++ (0,-3.1);
        \draw[black] (0,1) -- ++ (0.66,0) ;
        \draw[black, dotted] (0.66,1) -- ++ (0.66,0) ;
        \draw[black] (1.33,1) -- ++ (0.66,0) ;
        \draw[black] (2,1) -- ++ (1,0) ;
        \draw[black] (3,1) -- ++ (1,0) ;
        \draw[black] (4,1) -- ++ (0.66,0) ;
        \draw[black, dotted] (4.66,1) -- ++ (0.66,0) ;
        \draw[black] (5.33,1) -- ++ (0.66,0) ;
        \draw[black] (0,0) -- ++ (0.66,0) ;
        \draw[black, dotted] (0.66,0) -- ++ (0.66,0) ;
        \draw[black] (1.33,0) -- ++ (0.66,0) ;
        \draw[black] (2,0) -- ++ (0.5,0) ;
        \draw[black] (3.5,0) -- ++ (0.5,0) ;
        \draw[black] (4,0) -- ++ (0.66,0) ;
        \draw[black, dotted] (4.66,0) -- ++ (0.66,0) ;
        \draw[black] (5.33,0) -- ++ (0.66,0) ;
        \draw[black] (0,1) -- ++ (0,-1) ;
        \draw[black] (2,1) -- ++ (0,-1) ;
        \draw[black] (3,1) -- ++ (0,0.5) ;
        \draw[black] (4,1) -- ++ (0,-1) ;
        \draw[black] (6,1) -- ++ (0,-1) ;
        \draw[black] (3,1) -- ++ (0,-0.5) ;
        \node[draw,shape=circle,fill=Blue, scale=0.65] at (0,1){};
        \node[draw,shape=circle,fill=Blue, scale=0.65] at (2,1){};
        \node[draw,shape=circle,fill=Green, scale=0.65] at (3,1){};
        \node[draw,shape=circle,fill=Blue, scale=0.65] at (4,1){};
        \node[draw,shape=circle,fill=Blue, scale=0.65] at (6,1){};
        \node[draw,shape=semicircle,rotate=225,fill=Orange, anchor=south,inner sep=2pt, outer sep=0pt, scale=0.75] at (0,0){}; 
        \node[draw,shape=semicircle,rotate=45,fill=white, anchor=south,inner sep=2pt, outer sep=0pt, scale=0.75] at (0,0){};
        \node[draw,shape=semicircle,rotate=225,fill=Orange, anchor=south,inner sep=2pt, outer sep=0pt, scale=0.75] at (2,0){}; 
        \node[draw,shape=semicircle,rotate=45,fill=white, anchor=south,inner sep=2pt, outer sep=0pt, scale=0.75] at (2,0){};
        \node[draw,shape=semicircle,rotate=315,fill=white, anchor=south,inner sep=2pt, outer sep=0pt, scale=0.75] at (4,0){}; 
        \node[draw,shape=semicircle,rotate=135,fill=Orange, anchor=south,inner sep=2pt, outer sep=0pt, scale=0.75] at (4,0){};
        \node[draw,shape=semicircle,rotate=315 ,fill=white, anchor=south,inner sep=2pt, outer sep=0pt, scale=0.75] at (6,0){}; 
        \node[draw,shape=semicircle,rotate=135,fill=Orange, anchor=south,inner sep=2pt, outer sep=0pt, scale=0.75] at (6,0){};
        \draw[black] (2.5,-0.5) -- ++ (1,0) ;
        \draw[black] (3,-0.5) -- ++ (0,0.5) ;
        \node[draw,shape=circle,fill=Red, scale=0.65] at (3,-0.5){};
        \draw[Gray, ->, >=latex] (2.5,-0.4) -- ++ (0,0.3) ;
        \draw[Gray, ->, >=latex] (3.5,-0.4) -- ++ (0,0.3) ;
        \draw[Gray, ->, >=latex] (3,0.1) -- ++ (0,0.3) ;
        \node[inner sep=0, anchor=north] at (1,-0.66) {$\underbrace{\hspace*{2.2cm}}_{P_\mu}$};
        \node[inner sep=0, anchor=north] at (5,-0.66) {$\underbrace{\hspace*{2.2cm}}_{Q_\mu}$};
    \end{tikzpicture}
  \end{subfigure}%
  \begin{subfigure}{0.4\textwidth}
    \centering
    \caption{}
     \begin{tikzpicture}
       \draw[white] (2,1.65) -- ++ (0,-3.1);
       \draw[black] (-0.5,1) -- ++ (1,0) ;
       \draw[black] (0,1) -- ++ (0,0.5) ;
       \node[draw,shape=circle,fill=Red, scale=0.65] at (0,1){};
       \node[] at (1,1-0.025) {$=$};
       \def\x{2}
       \draw[black] (\x-0.5,1) -- ++ (2,0) ;
       \draw[black] (\x,1) -- ++ (0,0.5) ;
       \node[draw,shape=semicircle,rotate=225,fill=Orange, anchor=south,inner sep=2pt, outer sep=0pt, scale=0.75] at (\x,1){}; 
       \node[draw,shape=semicircle,rotate=45,fill=white, anchor=south,inner sep=2pt, outer sep=0pt, scale=0.75] at (\x,1){};
       \node[draw,shape=circle,fill=Gray, scale=0.65] at (\x+1,1){};
       \def\y{0}
       \draw[black] (-0.5,\y) -- ++ (1,0) ;
       \draw[black] (0,\y) -- ++ (0,0.5) ;
       \node[draw,shape=circle,fill=Red, scale=0.65] at (0,\y){};
       \node[] at (1,\y-0.025) {$=$};
       \def\x{2}
       \draw[black] (\x-0.5,\y) -- ++ (2,0) ;
       \draw[black] (\x+1,\y) -- ++ (0,0.5) ;
       \node[draw,shape=semicircle,rotate=315,fill=white, anchor=south,inner sep=2pt, outer sep=0pt, scale=0.75] at (\x+1,\y){};
       \node[draw,shape=semicircle,rotate=135,fill=Orange, anchor=south,inner sep=2pt, outer sep=0pt, scale=0.75] at (\x+1,\y){}; 
       \node[draw,shape=circle,fill=Gray, scale=0.65] at (\x,\y){};
    \end{tikzpicture}
  \end{subfigure}
  \caption{Construction and solution of the subproblem for the $\mu$th core: (a) The $4$-dimensional core of $\widehat{\PSI}_{X, \mu}$ (green circle) is contracted with the matrices $P_\mu$ and $Q_\mu$ constructed by joining the fixed cores of the coefficient tensor (orange circles) with the corresponding cores of the transformed data tensor. The matricization then defines the matrix $M_\mu$. (b) The TT core (red circle) obtained by solving the low-dimensional minimization problem is decomposed (e.g., using QR factorization) into a orthonormal tensor and a triangular matrix. The orthonormal tensor then yields the updated core.}
  \label{fig: ALS}
\end{figure}
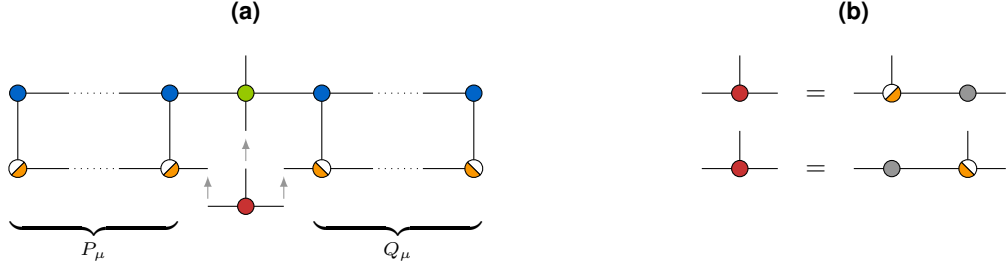

By orthonormalizing the fixed cores of $\XI$ and using truncated SVDs \cite{HANSEN1987} for solving the subsystems we can interpret our approach as a core-wise ridge regression approximating the solution obtained by kernel-based MANDy, see Appendix~\ref{app: Interpretation}. After approximating the coefficient tensor 
\begin{equation}\label{eq: ARR - coefficient tensor}
  \XI = \sum_{i=1}^{d^\prime} \XI_i \otimes e_i,
\end{equation}
the decision function $f$ is given by \eqref{eq:decision function}. The main difference between our approach and the method introduced in \cite{STOUDENMIRE2016a, STOUDENMIRE2016b} is that we do not update the TT cores of $\XI$ using gradient descent steps. Instead we solve a low-dimensional system of linear equations corresponding to the entire training data set whose solution yields the updated core. Moreover, we solve a minimization problem for each row of the label matrix $Y$. Using the modified basis decomposition introduced in Section~\ref{sec: basis decomposition}, it is possible to significantly reduce the storage consumption of the stack, see Algorithm~\ref{alg: ARR} Lines~\ref{alg: ARR - right stack 1} and \ref{alg: ARR - right stack 2}. If we only use a fixed representation for $\PSI_X$ as given in \eqref{eq: transformed data tensor}, the additional mode would lead to a much higher storage consumption of the right stack. Thus, our method provides an efficient construction of the subproblems. 

\section{Numerical results}
\label{sec:Numerical results}

We apply the tensor-based classification algorithms described in Sections~\ref{sec: Kernel-based MANDy} and \ref{sec: ALS approach} to both the MNIST and fashion MNIST data set, choosing the basis defined in \eqref{eq:Psi} and setting $ \alpha \approx 0.59 $. This value was determined empirically for the MNIST data set, but also leads to better classification rates for the fashion MNIST set. Kernel-based MANDy as well as ARR are available in Scikit-TT.\!\footnote{\url{https://github.com/PGelss/scikit_tt}} The numerical experiments were performed on a Linux machine with $128$ GB RAM and an Intel Xeon processor with a clock speed of $3$ GHz and $8$ cores. 

For the first approach, using kernel-based MANDy, we do not apply any regularization techniques. For the ARR approach, we set the TT rank for each solution $\XI_i$, see Algorithm~\ref{alg: ARR}, to $10$ and repeat the scheme $5$ times. Here, we use regularization, i.e., truncated SVDs with a relative threshold of $10^{-2}$ are applied to the minimization problems given in Algorithm~\ref{alg: ARR} (Lines~\ref{alg: ARR - subproblem 1} and \ref{alg: ARR - subproblem 2}). The obtained classification rates for the reduced and full MNIST and fashion MNIST data are shown in Figure~\ref{fig:MNIST_results}.

\begin{figure}
  \centering
  \begin{subfigure}{0.5\textwidth}
    \centering
    \caption{}
    \includegraphics[width=0.95\textwidth]{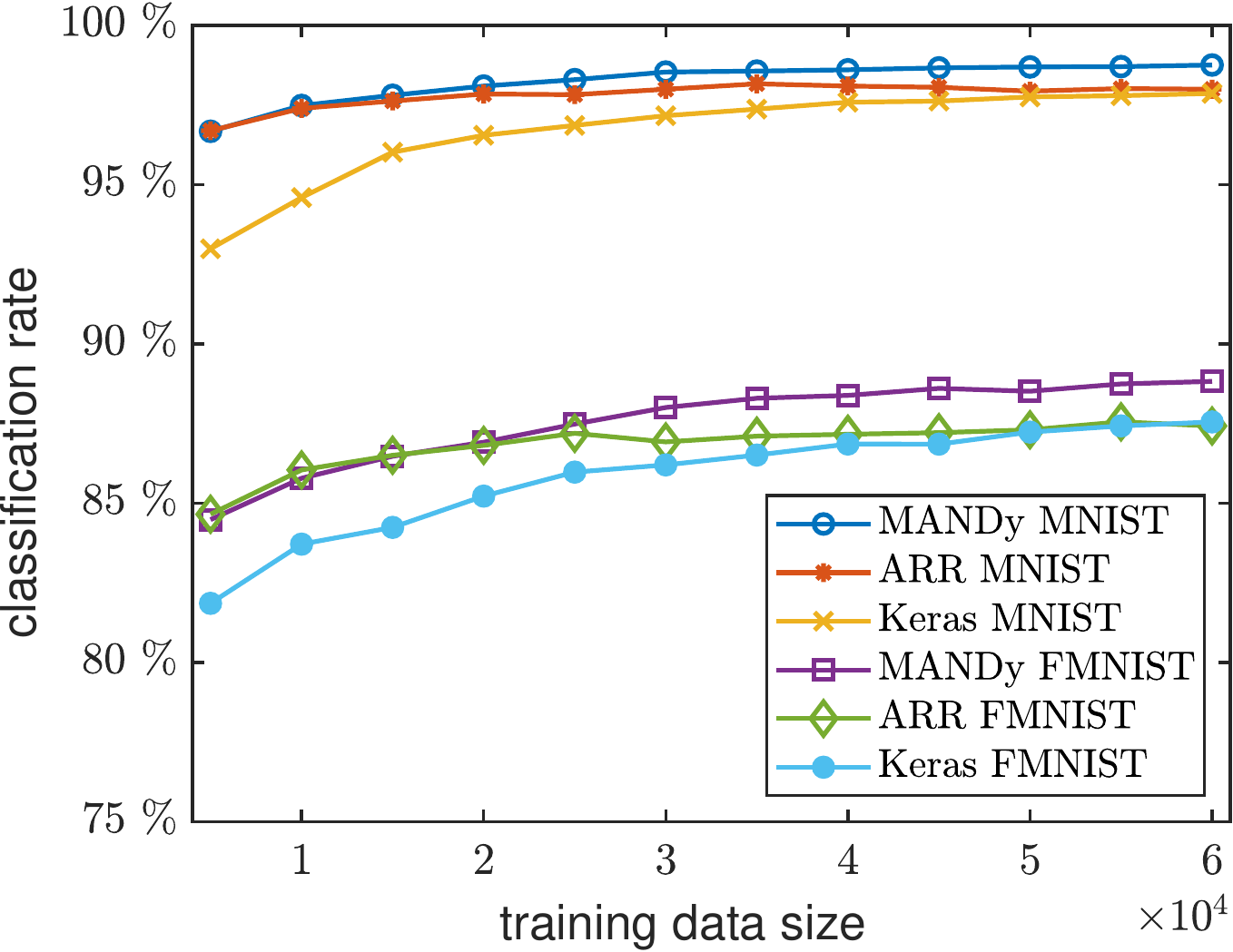}
  \end{subfigure}%
  \begin{subfigure}{0.5\textwidth}
    \centering
    \caption{}
    \includegraphics[width=0.95\textwidth]{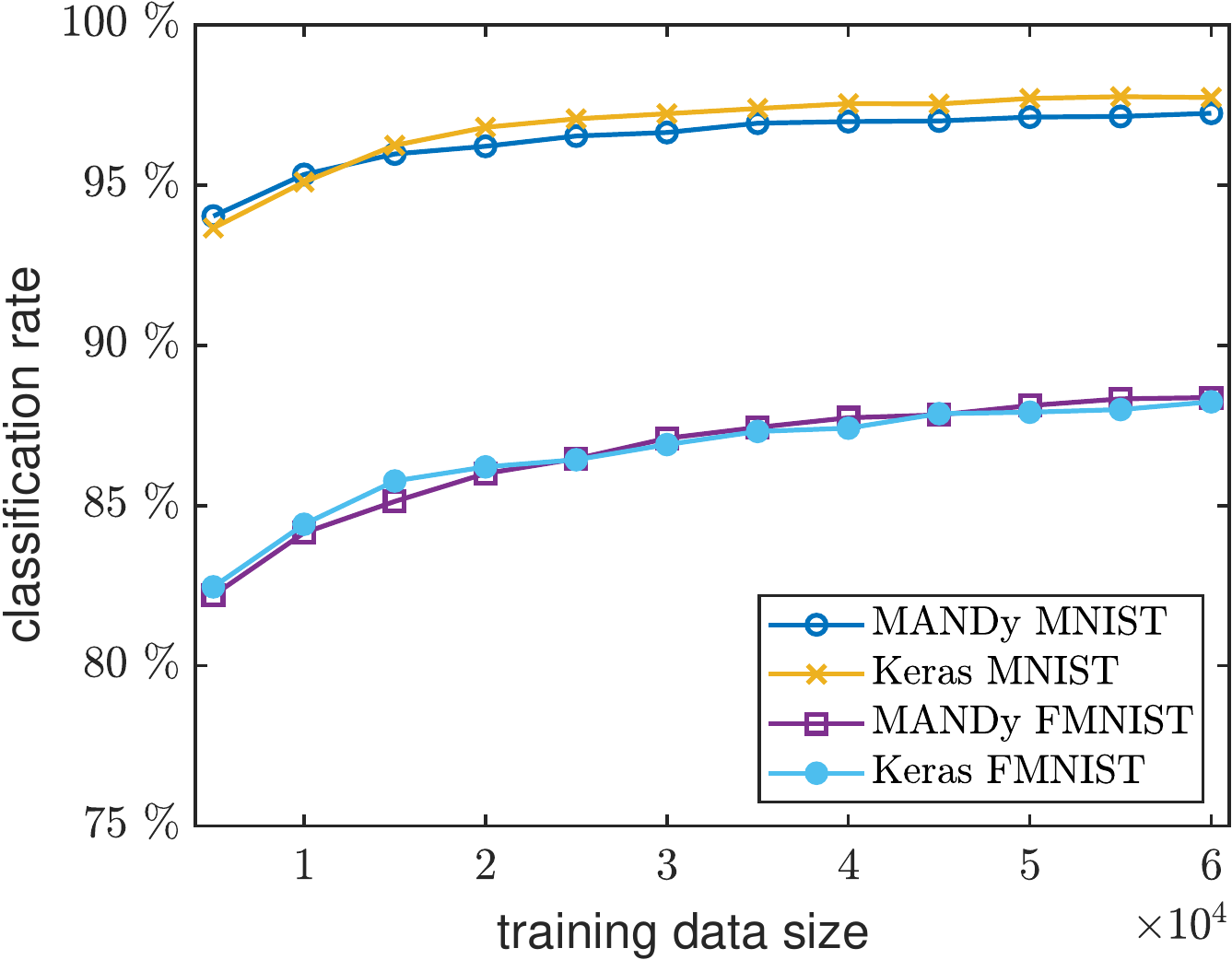}
  \end{subfigure}%
  \caption{Results for MNIST and fashion MNIST: (a) Classification rates for the reduced 14x14 images. (b) Classification rates for full 28x28 images. Reducing the image size by averaging over groups of pixels improves the performance of the algorithm.}
  \label{fig:MNIST_results}
\end{figure}

Similarly to \cite{STOUDENMIRE2016a, STOUDENMIRE2016b}, we first apply the classifiers to the reduced data sets, see Figure~\ref{fig:MNIST_results}\,(a). Using MANDy, we obtain classification rates of up to $ 98.75 \, \% $ for the MNIST and $ 88.82 \, \%$ for the fashion MNIST data set. Using the ARR approach, the classification rates are not monotonically increasing, which simply may be an effect of the alternating optimization scheme. The highest classification rates we obtain are $98.16 \, \%$ for the MNIST data and $87.55 \, \%$ for the fashion MNIST data. We typically obtain a $ 100 \, \% $ classification rate for the training data (as a consequence of the richness of the feature space). This is not necessarily a desired property since the learned model might not generalize well to new data, but seems to have no detrimental effects for the simple MNIST classification problem. As shown in Figure~\ref{fig:MNIST_results}\,(b), kernel-based MANDy can still be applied when considering the full data sets without reducing the image size. Here, we obtain classification rates of up to $97.24 \, \%$ for the MNIST and $88.37 \, \%$ for the fashion MNIST data set. That we obtain lower classification rates for the full images as compared to the reduced ones might be due to the fact that pixel-by-pixel comparisons of images are not expedient. The averaging effect caused by downscaling the images helps to detect coarser features. This is similar to the effect of convolutional kernels and pooling layers. In principle, ARR can also be used for the classification of the full data sets. So far, however, our numerical experiments produced only classification rates significantly lower than those obtained by applying MANDy ($95.94 \, \%$ for the MNIST and $82.18 \, \%$ for fashion MNIST data set). This might be due to convergence issues caused by the kernel. The application to higher-order transformed data tensors and potential improvements of ARR will be part of our future research.

Figure~\ref{fig:MNIST_results} also shows a comparison with \emph{tensorflow}. We run the code provided as a classification tutorial\footnote{\url{www.tensorflow.org/tutorials/keras/basic_classification}} ten times and compute the average classification rate. The input layer of the network comprises 784 nodes (one for each pixel; for the reduced data sets, we thus have only 196 input nodes), followed by two dense layers with 128 and 10 nodes. The layer with 10 nodes is the output layer containing probabilities that a given image belongs to the class represented by the respective neuron. Note that although more sophisticated methods and architectures for these problems exists---see the (fashion) MNIST website for a ranking---, the results show that our tensor-based approaches are competitive with state-of-the-art deep-learning techniques. 

\begin{figure}[htb]
    \centering
    \begin{subfigure}{0.28\textwidth}
        \centering
        \caption{}
        \includegraphics[width=\textwidth]{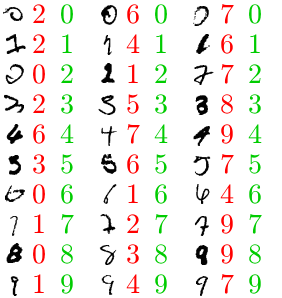}
    \end{subfigure}%
    \hspace{0.032\textwidth}
    \begin{subfigure}{0.328\textwidth}
        \centering
        \caption{}
        \includegraphics[width=\textwidth]{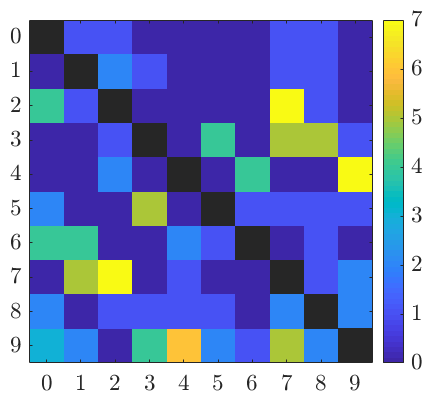} 
    \end{subfigure}%
    \\
    \begin{subfigure}{\textwidth}
        \centering
        \caption{}
        \includegraphics[width=0.64\textwidth]{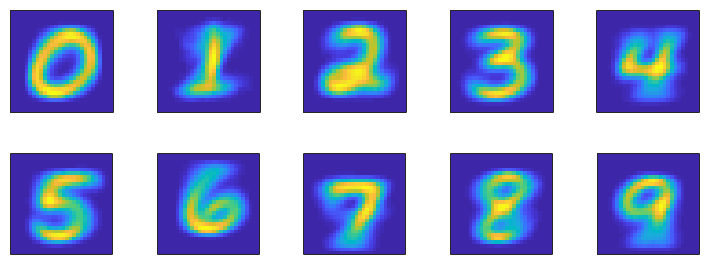}
    \end{subfigure}%
    \caption{MNIST classification: (a)~Images misclassified by kernel-based MANDy described in Section~\ref{sec: Kernel-based MANDy}. The original image is shown in black, the identified label in red, and the correct label in green. (b)~Histograms illustrating which categories are misclassified most often. The rows represent the correct labels of the misclassified image and the columns the detected labels. (c)~Visualizations of the learned classifiers showing a heat map of the classification function obtained by applying it to images that differ in one pixel.}
    \label{fig:MNIST_misclassified}
\end{figure}

In order to understand the numerical results for the MNIST data set (obtained by applying kernel-based MANDy to all $60\ts000$ training images), we analyze the misclassified images, examples of which are displayed in Figure~\ref{fig:MNIST_misclassified}\,(a). For misclassified images $x$, the entries of $ f(x) $, see~\eqref{eq:kernel classification}, are often numerically zero, which implies that there is no other image in the training set that is similar enough so that the kernel can pick up the resemblance. Some of the remaining misclassified digits are hard to recognize even for humans. Histograms demonstrating which categories are misclassified most often are shown in Figure~\ref{fig:MNIST_misclassified}\,(b). Here, we simply count the instances where an image with label $i$ was assigned the wrong label $j$. The digits $ 2 $ and $ 7 $ as well as $ 4 $ and $ 9 $ are confused most frequently. Additionally, we wish to visualize what the algorithm detects in the images. To this end, we perform a sensitivity analysis as follows: Starting with an image whose pixel values are constant everywhere (zero or any other value smaller than one, we choose $ 0.5 $), we set pixel $ (i, j) $ to one and compute $ y = f(x) $ for this image. The process is repeated for all pixels. For each label, we then plot a heat map of the values of $ y $. This tells us which pixels contribute most to the classification of the images. The resulting maps are shown in Figure~\ref{fig:MNIST_misclassified}\,(c). Except for the digit $ 1 $, the results are highly similar to the images obtained by averaging over all images containing a certain digit. 

\begin{figure}[htb]
    \centering
    \begin{subfigure}{0.28\textwidth}
        \centering
        \caption{}
        \includegraphics[width=\textwidth]{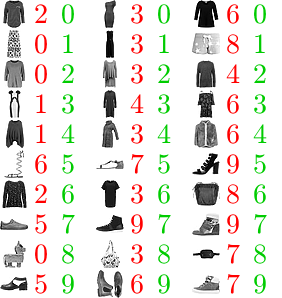}
    \end{subfigure}%
    \hspace{0.032\textwidth}
    \begin{subfigure}{0.328\textwidth}
        \centering
        \caption{}
        \includegraphics[width=\textwidth]{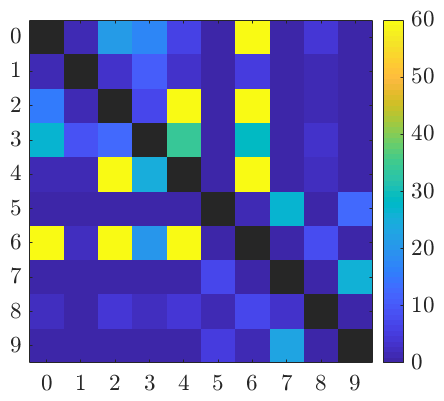} 
    \end{subfigure}%
    \\
    \begin{subfigure}{\textwidth}
        \centering
        \caption{}
        \includegraphics[width=0.64\textwidth]{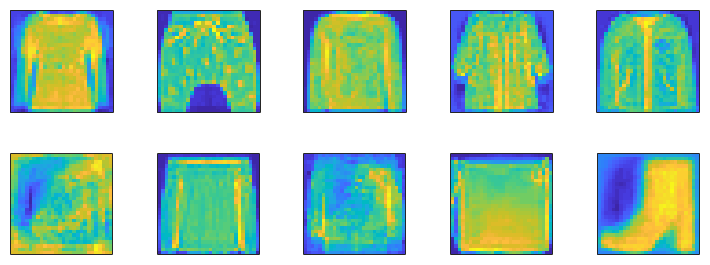}
    \end{subfigure}%
    \caption{Fashion MNIST classification: (a)~Misclassified images. (b)~Histogram of misclassified images. (c)~Visualizations of the learned classifiers.}
    \label{fig:FMNIST_misclassified}
\end{figure}

Figure~\ref{fig:FMNIST_misclassified} shows examples of misclassified images and the corresponding histogram as well as the results of the sensitivity analysis for the fashion MNIST data set. We see that the images of shirts (6) are most difficult to classify (due to the ambiguity in the category definitions), whereas trousers (1) and bags (8) have the lowest misclassification rates (probably due to their distinctive shapes). In contrast to the MNIST data set, the results of the sensitivity analysis differ widely from the average images. The classifier for coats (4), for instance, ``looks for'' a zipper and coat pockets, which are not visible in the \emph{average coat}, and the classifier for dresses (3) seems to base the decision on the presence of creases, which are also not distinguishable in the \emph{average dress}. The interpretation of other classifiers is less clear, e.g., the ones for sandals (5) and sneakers (7) seem to be contaminated by other classes.

Comparing the runtimes of both approaches applied to the reduced data sets with $60\ts000$ training images, kernel-based MANDy needs approximately one hour for the construction of the decision function \eqref{eq:kernel classification}. On the other hand, ARR needs less than $10$ minutes to compute the coefficient tensor assuming we parallelize Algorithm~\ref{alg: ARR}. 

\section{Conclusion}
\label{sec:Conclusion}

In this work, we presented two different tensor-based approaches for supervised learning. We showed that a kernel-based extension of MANDy can be utilized for image classification. That is, extending the method to arbitrary least-squares problems (originally, MANDy was developed to learn governing equations of dynamical systems) and using sequences of Hadamard products for the computation of the pseudoinverse, we were able to demonstrate the potential of kernel-based MANDy by applying it to the MNIST and fashion MNIST data sets. Additionally, we proposed the alternating optimization scheme ARR, which approximates the coefficient tensors by low-rank TT decompositions. Here, we used a mutable tensor representation of the transformed data tensors in order to construct low-dimensional regression problems for optimizing the TT cores of the coefficient tensor. 

Both approaches use an exponentially large set of basis functions in combination with least-squares regression techniques on a given set of training images. The results are encouraging and show that methods exploiting tensor products of simple basis functions are able to detect characteristic features in image data. The work presented in this paper constitutes a further step towards tensor-based techniques for machine learning.

The reason why we can handle the extremely high-dimensional feature space spanned by basis functions is its tensor-product format. Besides the general questions of the choice of basis functions and the expressivity of these functions, the rank-one tensor products that were used in this work can in principle be replaced by other structures which might result in higher classification rates. For instance, the transformation of an image could be given by a TT representation with higher ranks or hierarchical tensor decompositions (with the aim to detect features on different levels of abstraction). Furthermore, we could define different basis functions for each pixel, vary the number of basis functions per pixel, or define basis functions for groups of pixels. 

Even though kernel-based MANDy computes the minimum norm solution of the considered regression problems as an exact TT decomposition, the method is likely to suffer from high ranks of the transformed data tensors and might thus not be competitive for large data sets. At the moment, we are computing the Gram matrix for the entire training data set. However, a possibility to speed up computations and to lower the memory consumption would be to exploit properties of the kernel. That is, if the kernel almost vanishes if two images differ significantly in at least one pixel (as it is the case for the specific kernel used in this work, provided that the originally proposed value $ \alpha = \tfrac{\pi}{2}$ is used), the Gram matrix is essentially sparse when setting entries smaller than a given threshold to zero. Using sparse solvers would allow us to handle much larger data sets. Moreover, the construction of the Gram matrix is highly parallelizable and it would be possible to use GPUs to assemble it in a more efficient fashion.

Further modifications of ARR such as different regression methods for the subproblems, an optimized ordering of the TT cores, and specific initial coefficient tensors can help to improve the results. We provided an explanation for the stability of ARR, but the properties of alternating regression schemes have to be analyzed in more detail in the future.

\section*{Acknowledgments}

This research has been funded by Deutsche Forschungsgemeinschaft (DFG) through grant CRC 1114 \emph{``Scaling Cascades in Complex Systems''}. Part of this research was performed while S.K.~was visiting the Institute for Pure and Applied Mathematics (IPAM), which is supported by the National Science Foundation (Grant No. DMS-1440415). We would like to thank Michael G{\"o}tte and Alex Goe{\ss}mann from the TU Berlin for interesting discussions related to tensor decompositions and system identification.

\bibliographystyle{unsrturl}
\bibliography{references}

\newpage

\appendix

\section{Appendix}

\subsection{Representation of transformed data tensors}
\label{app: representation of transformed data tensors}

\begin{proposition}
  For all $i \in \{1, \dots, p\}$, it holds that
  \begin{equation*}
    \mat{\widehat{\PSI}_{X,i}}{n_1, \dots, n_d}{m} = \mat{\PSI_{X}}{n_1, \dots, n_d}{m}.
  \end{equation*}
  That is, the TT decompositions $\widehat{\PSI}_{X,i}$ and $\PSI_X$ represent the same tensor in $\R^{n_1 \times \dots \times n_p \times m}$.
\end{proposition}
\begin{proof}
  An entry of $\widehat{\PSI}_{X, \mu}$, $1 < \mu < p$, is given by
  \begin{equation*}
      \left( \widehat{\PSI}_{X, \mu} \right)_{i_1, \dots, i_p, j} = \sum_{k_1 = 1}^m \dots \sum_{k_{p-1} = 1}^m \left(\widehat{\PSI}_{X, \mu}^{(1)} \right)_{1, i_1, k_1} \cdot \ldots \cdot \left(\widehat{\PSI}_{X, \mu}^{(\mu)} \right)_{k_{\mu-1}, i_\mu,j, k_\mu} \cdot \ldots \cdot \left(\widehat{\PSI}_{X, \mu}^{(p)} \right)_{k_{p-1}, i_p, 1}.
  \end{equation*}    
  By definition, 
  \begin{equation*}
    \left(\widehat{\PSI}_{X, \mu}^{(\mu)} \right)_{k_{\mu-1}, i_\mu,j, k_\mu} \neq 0 \quad \Leftrightarrow \quad k_{\mu-1} = j = k_\mu.
  \end{equation*}
  On the other hand, an entry of $\widehat{\PSI}_{X, \mu}^{(\nu)}$ with $\nu \neq \mu$ and $1 < \nu < p$ is nonzero if and only if $k_{\nu-1} = k_\nu$. It follows that
  \begin{equation*}
    \begin{split}
      \left( \widehat{\PSI}_{X, \mu} \right)_{i_1, \dots, i_p, j} &= \left(\widehat{\PSI}_{X, \mu}^{(1)} \right)_{1, i_1, j} \cdot \ldots \cdot \left(\widehat{\PSI}_{X, \mu}^{(\mu)} \right)_{j, i_\mu,j, j} \cdot \ldots \cdot \left(\widehat{\PSI}_{X, \mu}^{(p)} \right)_{j, i_p, 1}\\
      &= \psi_{1,i_1}(x_j) \cdot \ldots \cdot \psi_{\mu,i_\mu}(x_j)  \cdot \ldots \cdot \psi_{p,i_p}(x_j),
    \end{split}
  \end{equation*}
  This can be shown in an analogous fashion for $\mu=1$ and $\mu=p$.
\end{proof}

\subsection{Interpretation of ARR as ALS ridge regression}
\label{app: Interpretation}

The following reasoning will elucidate the relation between ARR, ridge regression, and kernel-based MANDy. We only outline the rough idea without concrete proofs. Let $R_\mu$ denote the retraction operator, see \cite{HOLTZ2012}, consisting of the fixed TT cores $\XI^{(1)}, \dots, \XI^{(\mu-1)}$ and $\XI^{(\mu+1)}, \dots, \XI^{(p)}$ of the solution $\XI$ at any iteration step of Algorithm~\ref{alg: ARR}. Furthermore, assume that $\XI^{(1)}, \dots, \XI^{(\mu-1)}$ are left- and $\XI^{(\mu+1)}, \dots, \XI^{(p)}$ right-orthonormal. In Lines~\ref{alg: ARR - subproblem 1} and \ref{alg: ARR - subproblem 2} of Algorithm~\ref{alg: ARR}, we consider the system (with a slight abuse of notation)
\begin{equation*}
  y = M_\mu x = \left(\PSI_X^\top \cdot R_\mu \right)  x.
\end{equation*}
The application of a truncated SVD to the matricization of $\PSI_X^\top \cdot R_\mu$ (as done in Algorithm~\ref{alg: ARR}) is then similar to a regularization in the form of
\begin{equation}\label{eq: Tikhonov}
  \min_x \left\{ \norm{y- M_\mu x}_2^2 + \varepsilon \norm{x}_2^2 \right\}
\end{equation}
with appropriate regularization parameter $\varepsilon$, i.e., $x \approx M_\mu^+ y$ for both approaches, see \cite{HANSEN1987, GROETSCH1993}. The formulation~\eqref{eq: Tikhonov} is known as Tikhonov's smoothing functional, ridge regression, or $\ell^2$ regularization (which, of course, could also directly be applied in Algorithm~\ref{alg: ARR}). The solution of \eqref{eq: Tikhonov} is also the solution of the regularized normal equation
\begin{equation*}
  M_\mu^\top y = \left( M_\mu^\top M_\mu + \varepsilon \ts \Id \right)x,
\end{equation*}
see, e.g., \cite{ZHDANOV2012}. Since $R^\top_\mu  R_\mu = \Id$, it follows that
\begin{equation*}
  \left( R^\top_\mu  \PSI_X \right)  y = \left( R^\top_\mu \left( \PSI_X \PSI_X^\top + \varepsilon \ts \Id \right) R_\mu \right)  x.
\end{equation*}
In fact, this is a subproblem corresponding to the application of ALS \cite{HOLTZ2012} to the tensor-based system
\begin{equation}\label{eq: ALS normal equation}
  \PSI_X y = \left( \PSI_X \PSI_X^\top + \varepsilon \ts \Id \right) \XI.
\end{equation}
Note that all requirements for the application of ALS are satisfied since $\PSI_X \PSI_X^\top + \varepsilon \ts \Id $ is a symmetric positive definite tensor operator and $R_\mu$ is orthonormal. The system of linear equations given in \eqref{eq: ALS normal equation} is then equivalent to the minimization problem
\begin{equation*}
  \min_{\XI} \left\{ \norm{y - \PSI_X^T \XI}_2^2 + \varepsilon \norm{\XI}_2^2 \right\}.
\end{equation*}
For sufficiently small $\varepsilon$, it holds that $\XI \approx \PSI_X^+ y$, see \cite{BARATA2012}, meaning Algorithm~\ref{alg: ARR} computes an approximation of the coefficient tensor resulting from the application of kernel-based MANDy, see Section~\ref{sec: Kernel-based MANDy}.

\end{document}